\documentclass[11pt]{article}
\usepackage{fullpage}

\title{Multi-stage Convex Relaxation for Feature Selection}

\author{Tong Zhang \\
  Statistics Department \\
  Rutgers University \\
  Piscataway, NJ   08854 \\
  \texttt{\small tzhang@stat.rutgers.edu}
}

\date{}
\usepackage{amsmath}
\usepackage{amssymb}
\usepackage{graphicx} 

\usepackage{fancybox,graphicx}
\usepackage{float}
\usepackage{url}
\floatstyle{ruled}
\newfloat{alg}{tbp}{loa}
\floatname{alg}{Algorithm}

{
\rm
\begin{tabbing} 
....\=...\=...\=...\=...\=  \+ \kill
} %
{\end{tabbing}
}

\newtheorem{lemma}{Lemma}
\newtheorem{theorem}{Theorem}
\newtheorem{definition}{Definition}

\newtheorem{assumption}{Assumption}

\newcommand{\bx}{{\mathbf x}}

\newcommand{\by}{{\mathbf y}}

\newcommand{\bw}{{\mathbf w}}
\newcommand{\bv}{{\mathbf v}}
\newcommand{\bu}{{\mathbf u}}

\newcommand{\rE}{{\mathbf E}}
\newcommand{\sgn}{{\mathrm{sgn}}}

\newcommand{\Fr}{{\mathrm{supp}}}

\newcommand{\R}{{\mathbb R}}

\newcommand{\BlackBox}{\rule{1.5ex}{1.5ex}}  
\newenvironment{proof}{\par\noindent{\bf Proof\ }}{\hfill\BlackBox\\[2mm]}

\begin{document}

\maketitle

\begin{abstract}
A number of recent work studied the effectiveness of feature selection using Lasso. 
It is known that under the restricted isometry properties (RIP), Lasso does not generally lead 
to the exact recovery of the set of nonzero coefficients, due to the looseness of convex relaxation.
This paper considers the feature selection property of nonconvex regularization, where
the solution is given by a multi-stage convex relaxation scheme.
Under appropriate conditions, we show that the local solution obtained by this procedure recovers
the set of nonzero coefficients without suffering from the bias of Lasso relaxation, which complements
parameter estimation results of this procedure in \cite{Zhang09-multistage}.
\end{abstract}

\section{Introduction}

We consider the linear regression problem, where we observe a set of
input vectors $\bx_1, \ldots, \bx_n \in R^p$, with corresponding 
desired output variables $y_1, \ldots, y_n$. 
In a statistical linear model, it is common to assume that there exists a target coefficient vector
$\bar{\bw} \in R^p$ such that 
\begin{equation}
 y_i = \bar{\bw}^\top \bx_i  + \epsilon_i
\qquad (i=1,\ldots,n) , \label{eq:noise-def}
\end{equation}
where $\epsilon_i$ are zero-mean independent random noises 
(but not necessarily identically distributed). 
Moreover, we assume that the target vector $\bar{\bw}$ is sparse. That is,
$\bar{k}=\|\bar{\bw}\|_0$ is small. Here we use the standard notation
\[
\Fr(\bw)= \{j: \bw_j \neq 0\}  \qquad
\|\bw\|_0 = |\Fr(\bw)| 
\]
for any vector $\bw \in R^p$.

This paper focuses on the feature selection problem, where we are interested in estimating the set of
nonzero coefficients $\Fr(\bar{\bw})$ (also called support set).
Let $\by$ denote the vector of $[y_i]$ and $X$ be the $n \times d$ matrix
with each row a vector $\bx_i$. The standard statistical method is subset selection ($L_0$ regularization), which
computes the following estimator 
\begin{equation}
\hat{\bw}_{L_0} = \arg\min_{\bw \in R^p} \|X \bw - \by\|_2^2 
\qquad \text{subject to } \|\bw\|_0 \leq k , \label{eq:L0} 
\end{equation}
where $k$ is a tuning parameter. This method is arguably a natural method for feature selection because 
if noise $\epsilon_i$ are iid Gaussian random variables, then (\ref{eq:L0}) can be regarded as a Bayes procedure
with an appropriately defined sparse prior over $\bw$. 
However, because the optimization problem in (\ref{eq:L0}) 
is nonconvex, the global solution of this problem cannot be efficiently computed. 
In practice, one can only find an approximate solution of (\ref{eq:L0}). The most popular approximation to $L_0$
regularization is the $L_1$ regularization method 
which is often referred to as Lasso \cite{Tib96}:
\begin{equation}
\hat{\bw}_{L_1} = \arg\min_{\bw \in R^p} \left[ \frac{1}{n} \|X \bw - \by\|_2^2 + \lambda \|\bw\|_1 \right] ,
\label{eq:L1}
\end{equation}
where $\lambda>0$ is an appropriately chosen regularization parameter. 

The global optimum of (\ref{eq:L1}) can be easily computed using standard convex programming techniques. 
It is known that in practice, $L_1$ regularization often leads to sparse solutions (although often suboptimal). Moreover,
its performance has been theoretically analyzed recently. 
For example, it is known from the compressed sensing literature (e.g., \cite{CandTao05-rip}) that
under certain conditions referred to as {\em restricted isometry property} (RIP), 
the solution of $L_1$ relaxation (\ref{eq:L1}) approximates the solution of the $L_0$ regularization problem (\ref{eq:L0}).
The prediction and parameter performance of this method has been considered in \cite{BunTsyWeg07,BiRiTs07,Koltchinskii08,ZhangHuang06,Zhang07-l1,GeerBuhlmann09-conditions}.
Exact support recovery was considered by various authors such as \cite{MeinBuh06,ZhaoYu06,Wainwright06}.
It is known that under some more restrictive conditions referred to as {\em irrepresentable conditions}, $L_1$ regularization can achieve exact recovery of the support set.
However, the $L_1$ regularization method (\ref{eq:L1}) does not achieve exact recovery of the support set under the RIP type of conditions, which we are interested in here. 

Although it is possible to achieve exact recovery using post-processing by thresholding the small coefficients of Lasso solution, 
this method is suboptimal under RIP in comparison to the $L_0$ regularization method (\ref{eq:L0})
because it requires the smallest nonzero coefficients to be $\sqrt{\bar{k}}$ times larger than the noise level instead of
only requiring the nonzero coefficients to be larger than the noise level with $L_0$ regularization in (\ref{eq:L0}). This issue, referred to as the {\em bias} of Lasso for feature selection, was extensively discussed in \cite{Zhang10-mc+}.
Detailed discussion can be found after Theorem~\ref{thm:multi-stage-featsel}.
It is worth mentioning that under a stronger mutual coherence condition (similar to irrepresentable condition), this post-processing step does not give this bias factor $\sqrt{\bar{k}}$ as shown in \cite{Lounici08} (also see \cite{Zhang07-l1}). Therefore the advantage of bias removal for the multi-stage procedure discussed here is only applicable when RIP holds but when the irrepresentable condition and mutual incoherence conditions fail.
A thorough discussion of various conditions is beyond the scope of the current paper, and
we would like to refer the readers to \cite{GeerBuhlmann09-conditions}.
Nevertheless,
it is worth pointing out that even in the classical $d<n$ setting with the design matrix $X$ being rank $d$, the irrepresentable condition or the mutual incoherence condition can still be violated while the RIP type sparse-eigenvalue condition used in this paper holds trivially. In fact, this was pointed out in \cite{Zou06} as the main motivation of adaptive Lasso. Adaptive Lasso behaves similarly to the above mentioned post-processing, and thus suffers from the same bias problem.

The bias of Lasso is due to the looseness of convex relaxation for $L_0$ regularization. Therefore the remedy is to use
a non-convex regularizer that is close to $L_0$ regularization. One drawback of using nonconvex optimization formulation is that
we can only find a local optimal solution and different computational procedure may lead to a different local solution. 
Therefore the theoretical analysis has to be integrated with specific computational procedure to show that the local minimum
obtained by the procedure has desirable properties (e.g., exact support recovery). 
Several nonconvex computational procedures have been analyzed in the literature,
including an adaptive forward backward greedy procedure (referred to as FoBa) 
to approximately solve the regularization method (\ref{eq:L0}) considered in
\cite{Zhang08-foba}, and the MC+ method in \cite{Zhang10-mc+} to solve a non-convex regularized problem using
a path-following procedure. Both methods can achieve unbiased feature selection. 

Related to the above mentioned work, a different procedure, referred to as {\em multi-stage convex relaxation}, was analyzed in
\cite{Zhang09-multistage}. This procedure solves a nonconvex problem using multiple stages of Lasso relaxations,
where convex formulations are iteratively refined based on solutions obtained from the previous stages. 
However, only parameter estimation performance was analyzed in \cite{Zhang09-multistage}. Unfortunately, the result
in  \cite{Zhang09-multistage} does not directly imply that multi-stage convex relaxation achieves unbiased recovery of the
support set. The purpose of this paper is to prove such a support recovery result analogous to related result in \cite{Zhang10-mc+}
(which is for a different procedure), and this result complements the parameter estimation result of \cite{Zhang09-multistage}.

\section{Multi-Stage Convex Relaxation with Capped-$L_1$ Regularization}
\label{sec:theory}

We are interested in recovering $\bar{\bw}$ from noisy observations $\by$ using the following nonconvex
regularization formulation:
\begin{equation}
\hat{\bw} = \arg\min_{\bw} \left[ \frac{1}{n}\| X \bw - \by \|_2^2 + \lambda \sum_{j=1}^p g(|\bw_j|) \right] , \label{eq:sparse-reg}
\end{equation}
where $g(|\bw_j|)$ is a regularization function. 
For simplicity, this paper only considers the specific regularizer
\begin{equation}
g(u) = \min(u,\theta) , \label{eq:capped-L1-reg}
\end{equation}
which is referred to as capped-$L_1$ regularization in \cite{Zhang09-multistage}. The parameter $\theta$ is a thresholding
parameter which says that we use $L_1$ penalization when a coefficient is sufficiently small, but the penalty does not increase
when the coefficient is larger than a threshold $\theta$. Detailed discussions can be found in \cite{Zhang09-multistage}.
Similar to \cite{Zhang09-multistage}, one can analyze general regularization function $g(u)$. However, some of such functions (such as adaptive Lasso) do not completely remove the bias. Therefore we only analyze the simple function (\ref{eq:capped-L1-reg}) in this paper for clarity. 
While a theoretical justification has been given in \cite{Zhang09-multistage} for multi-stage convex relaxation,
similar procedure has been shown to work well empirically without theoretical
justification \cite{CanWakBoy08,WipfNaga10}. 
Moreover, a two-stage version was proposed in \cite{ZouLi08}, which
does not remove the bias issue discussed in this paper.

Since the regularizer (\ref{eq:capped-L1-reg}) is nonconvex, the resulting optimization problem 
(\ref{eq:sparse-reg}) is a non-convex regularization problem. 
However the regularizer in (\ref{eq:capped-L1-reg}) is continuous and piecewise differentiable, and thus 
its solution is easier to compute than the $L_0$ regularization method in (\ref{eq:L0}). For example,  standard numerical
techniques such as sub-gradient descent lead to local minimum solutions.
Unfortunately, it is difficult to find the global optimum, and it is also
difficult to analyze the quality of the local minimum obtained from the gradient descent method.
As a matter of fact, results with non-convex regularization are difficult
to reproduce because different numerical optimization procedures can
lead to different local minima. Therefore the quality of the solution
heavily depend on the numerical procedure used. 

In the following, we consider a specific numerical procedure referred to as multi-stage convex relaxation 
in \cite{Zhang09-multistage}. 
The algorithm is given in Figure~\ref{fig:multi-stage-sparse}. 
The procedure converges to a local optimal solution of (\ref{eq:sparse-reg}) due to a simple concave duality argument,
where (\ref{eq:sparse-reg}) is rewritten as
\[
\hat{\bw} = \arg\min_{\bw} \min_{\{\lambda_j \geq 0\}}\left[ \frac{1}{n}\| X \bw - \by \|_2^2 + 
\sum_{j=1}^p \lambda_j |\bw_j| + \sum_{j=1}^p g^*(\lambda_j) \right] , 
\]
with $g^*(\lambda_j)=\max((\lambda-\lambda_j)\theta,0)$.
The procedure  of Figure~\ref{fig:multi-stage-sparse} can be regarded as an alternating optimization method to
solve this joint optimization problem of $\bw$ and $\{\lambda_j\}$, where the first step solves for $\bw$ with
$\{\lambda_j\}$ fixed, and the second step is the closed form solution of $\{\lambda_j\}$ with $\bw$ fixed.
A more detailed discussion can be found in \cite{Zhang09-multistage}.
Our goal is to show that this procedure
can achieve unbiased feature selection as described in \cite{Zhang10-mc+}.

\begin{figure}[ht]
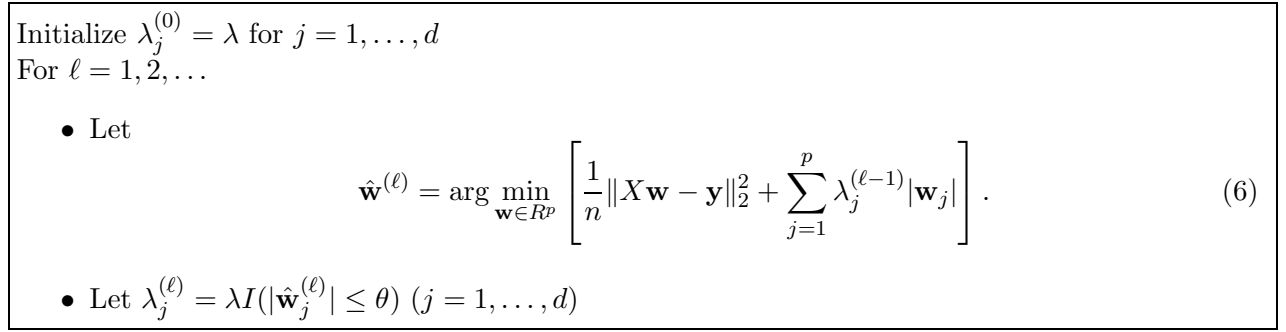

  \centering
  \begin{Sbox}
    \begin{minipage}{\linewidth}
      Initialize $\lambda_j^{(0)}=\lambda$ for $j=1,\ldots,d$ \\
      For $\ell=1,2,\ldots$
      \begin{itemize}
      \item Let
        \begin{equation}
          \hat{\bw}^{(\ell)} = \arg\min_{\bw \in R^p} \left[ \frac{1}{n}\|X \bw - \by\|_2^2 + 
            \sum_{j=1}^p  \lambda_j^{(\ell-1)} |\bw_j| \right] .
          \label{eq:convex-relax-L1}
        \end{equation}
    \item Let $\lambda_j^{(\ell)} = \lambda I(|\hat{\bw}_j^{(\ell)}| \leq \theta )$ ($j=1,\ldots,d$)
\end{itemize}
\end{minipage}
\end{Sbox}\fbox{\TheSbox}
\caption{Multi-stage Convex Relaxation for Sparse Regularization}
\label{fig:multi-stage-sparse}
\end{figure}

\section{Theoretical Analysis}

We require some technical conditions for our analysis.
First we assume sub-Gaussian noise as follows.
\begin{assumption}\label{assump:fixed}
Assume that
$\{\epsilon_i\}_{i=1,\ldots,n}$ in (\ref{eq:noise-def}) are independent (but not necessarily identically distributed) sub-Gaussians:
there exists $\sigma \geq 0$ such that $\forall i$ and $\forall t \in R$,
\[
\rE_{\epsilon_i} e^{t \epsilon_i} \leq e^{\sigma^2 t^2/2} .
\]
\end{assumption}
Both Gaussian and bounded random variables are sub-Gaussian
using the above definition. 
For example, if a random variable 
$\xi \in [a,b]$, then $\rE_\xi e^{t (\xi - \rE \xi)} \leq e^{(b-a)^2 t^2/8}$.
If a random variable is Gaussian: $\xi \sim N(0,\sigma^2)$, then
$\rE_\xi e^{t \xi} \leq e^{\sigma^2 t^2/2}$.

We also introduce the concept of sparse eigenvalue, which is standard in
the analysis of $L_1$ regularization. 
\begin{definition}
Given $k$, define
  \begin{align*}
  \rho_+(k)=&\sup \left\{\frac{1}{n}\|X \bw\|_2^2/\|\bw\|_2^2 : \|\bw\|_0 \leq k \right\} ,\\
  \rho_-(k)=&\inf \left\{\frac{1}{n}\|X \bw\|_2^2/\|\bw\|_2^2 : \|\bw\|_0 \leq k \right\} .
\end{align*}
\end{definition}

The following result for parameter estimation was obtained in \cite{Zhang09-multistage},
under the Assumption~\ref{assump:fixed}. If we assume that the target $\bar{\bw}$ is sparse, with
$\rE y_i = \bar{\bw}^\top \bx_i$, and $\bar{k}=\|\bar{\bw}\|_0$, and we
choose $\theta$ and $\lambda$ such that
\[
\lambda \geq 20 \sigma \sqrt{2\rho_+(1) \ln (2p/\eta)/n} 
\]
and 
\[
\theta \geq 9 \lambda /\rho_-(2\bar{k}+s) .
\]
Assume that
$\rho_+(s)/\rho_-(2\bar{k}+2s) 
\leq 1+ 0.5 s/\bar{k}$ for some $s \geq 2 \bar{k}$, then
with probability larger than $1-\eta$:
\begin{equation}
  \|\hat{\bw}^{(\ell)} - \bar{\bw}\|_2
  \leq 
  \frac{17}{\rho_-(2\bar{k}+s)}
  \left[ 2 \sigma \sqrt{\rho_+(\bar{k})} \left(\sqrt{\frac{7.4 \bar{k}}{n}} +
      \sqrt{\frac{2.7 \ln (2/\eta)}{n}} \right) 
+\lambda \sqrt{k_\theta} \right]
+ \frac{0.7^{\ell} \cdot \sqrt{\bar{k}} \lambda}{\rho_-(2 \bar{k}+s)} ,
\label{eq:multi-stage-param-est-bound}
\end{equation}
where $\hat{\bw}^{(\ell)}$ is the solution of
(\ref{eq:convex-relax-L1}), and
$k_\theta = \left|\{j \in \bar{F}: |\bar{\bw}_j| \leq 2\theta\}\right|$.

The condition   $\rho_+(s)/\rho_-(2\bar{k}+2s) \leq 1+ 0.5 s/\bar{k}$ 
requires the eigenvalue ratio $\rho_+(s)/\rho_-(s)$ to
grow sub-linearly in $s$. Such a condition, referred to as 
{\em sparse eigenvalue condition}, 
is also needed in the standard analysis of $L_1$ regularization
\cite{ZhangHuang06,Zhang07-l1}. It is related but  slightly weaker than the 
RIP condition in compressive sensing \cite{CandTao05-rip}, which requires the condition
\[
1 - \delta_{s'} \leq \rho_-(s') \leq \rho_+(s') \leq 1 + \delta_{s'} ,
\] 
for some $\delta_{s'} \in (0,1)$ and $s'>\bar{k}$. For example, with $s'=6\bar{k}$, and the restricted isometry constant
$\delta_{s'} \leq 1/3$, then the sparse eigenvalue condition above holds with $s=2\bar{k}$.
For simplicity, in this paper we do not make distinctions between RIP and sparse eigenvalue
condition.
Note that in the traditional low-dimensional statistical
analysis, one assumes that $\rho_+(s)/\rho_-(2\bar{k}+2s)< \infty$ as 
$s \to \infty$, which is significantly stronger than the condition we
use here. Although in practice it is often difficult to verify 
the sparse eigenvalue condition for real problems,
the parameter estimation result in (\ref{eq:multi-stage-param-est-bound})
nevertheless provides important theoretical
insights for multi-stage convex relaxation. 

For standard Lasso, we have the following bound 
\[
\|\hat{\bw}_{L_1}- \bar{\bw}\|_2 =O(\sqrt{k} \lambda) ,
\]
where $\hat{\bw}_{L_1}$ is the solution of the standard $L_1$ regularization.
This bound is tight for Lasso, 
in the sense that the right hand side cannot be improved
except for the constant---this
can be easily verified with an orthogonal design matrix. 
It is known that in order for Lasso to be effective, one has to
pick $\lambda$ no smaller than the order $\sigma \sqrt{\ln p/n}$.
Therefore, the parameter estimation error of the standard Lasso is of the order
$\sigma \sqrt{\bar{k} \ln p/n}$, which cannot be improved.

In comparison, if we consider the capped-$L_1$ regularization with
$g(|\bw_j|)$ defined in (\ref{eq:capped-L1-reg}), 
the bound in (\ref{eq:multi-stage-param-est-bound}) can
be significantly better
when most non-zero coefficients of $\bar{\bw}$ are
relatively large in magnitude. In the extreme case where $k_\theta = \left|\{j: |\bar{\bw}_j| \in (0,2\theta]\}\right|=0$,
which can be achieved when all nonzero components of $\bar{\bw}$ are larger than
the order $\sigma \sqrt{\ln p/n}$, we obtain the following better bound
\[
\|\hat{\bw}^{(\ell)}- \bar{\bw}\|_2 = O(\sqrt{\bar{k}/n} + \sqrt{\ln (1/\eta)/n})
\]
for the multi-stage procedure for a sufficiently large $\ell$ at the order of $\ln k + \ln \ln p$. 
This bound is superior to the standard one-stage $L_1$ regularization bound
$\|\hat{\bw}_{L_1}- \bar{\bw}\|_2 = O(\sqrt{\bar{k} \ln (p/\eta)/n})$.
 
In the literature, one is often interested in two types of results, one is parameter estimation 
bound as in (\ref{eq:multi-stage-param-est-bound}), and the other is feature selection consistency:
that is, to identify the set of nonzero coefficients of the truth. 
Although the parameter estimation bound in (\ref{eq:multi-stage-param-est-bound}) 
is superior to Lasso, the result does not imply that one can correctly select all variables under this condition.
Moreover, the specific proof presented in \cite{Zhang09-multistage} does not directly imply
such a result. Therefore it is important to know whether the multi-stage convex relaxation can 
achieve unbiased feature selection as studied in \cite{Zhang10-mc+}. 
In the following, we present such a result which supplements the parameter estimation bound
of (\ref{eq:multi-stage-param-est-bound}). While the main high-level argument
follows that of \cite{Zhang09-multistage}, 
there are many differences in the details, and hence a full proof (which is included in Section~\ref{apx:proof})
is still needed. This theorem is the main result of the paper.
It is worth mentioning that although we only consider the simple capped-$L_1$ regularizer, similar results
can be obtained for other regularizers (with virtually the same proof) such that $g'(u) \in [0, \infty)$, $g'(u)>0$ when $u$ belongs to a neighbor of $0$, and $g'(u)=0$ when $u \geq \theta$,
with a threshold $\theta>0$ appropriately chosen at the order of the noise level --- the condition of
$g'(u)=0$ when $u \geq \theta$ ensures the removal of feature selection ``bias'' of Lasso
which we discussed above.
As an example, very similar result can be obtained for
the MC+ penalty of \cite{Zhang10-mc+} or SCAD penalty of \cite{FanLi01}
using the multi-stage convex relaxation procedure here.
In fact, in practice there may be additional advantages of using a smooth nonconvex penalty such as
MC+ due to the extra smoothness, although such advantage is not revealed in our theoretical analysis.

\begin{theorem}\label{thm:multi-stage-featsel}
  Let Assumption~\ref{assump:fixed} hold.
  Assume also that the target $\bar{\bw}$ is sparse, with
  $\rE y_i = \bar{\bw}^\top \bx_i$, and $\bar{k}=\|\bar{\bw}\|_0$. 
  Let $\bar{F}=\Fr(\bar{\bw})$.
  Choose $\theta$ and $\lambda$ such that
  \[
  \lambda \geq 7 \sigma \sqrt{2\rho_+(1) \ln (2p/\eta)/n} 
  \]
  and
  \[
  \theta > 9 \lambda /\rho_-(1.5\bar{k}+s) .
  \]
  Assume that
  \[
  \min_{j \in \bar{F}} |\bar{\bw}_j| > 2 \theta
  \]
  and 
  $\rho_+(s)/\rho_-(1.5\bar{k}+2s) \leq 1+ 2s/(3\bar{k})$ for some $s \geq 1.5 \bar{k}$, then
  with probability larger than $1-\eta$:
  \[
  \Fr(\hat{\bw}^{(\ell)})=\Fr(\bar{\bw})
  \]
  when $\ell>L$, where $\hat{\bw}^{(\ell)}$ is the solution of (\ref{eq:convex-relax-L1}) and
  \[
  L = \left\lfloor \frac{0.5 \ln \bar{k} }{\ln (\rho_-(1.5\bar{k}+s)\theta/(6\lambda))}    \right\rfloor +1 .
  \]
\end{theorem}

Theorem~\ref{thm:multi-stage-featsel} is the main result of this paper. If
\[
\min_{\bw_j \in \bar{F}} |\bw_j| \geq c \sigma\sqrt{\ln p/n}
\]
for a sufficiently large constant $c$ that is independent of $\bar{k}$
(but could depend on the RIP condition), 
then we can pick both parameters $\lambda=O(\sigma \sqrt{\ln p/n})$ and $\theta=O(\sigma \sqrt{\ln p/n})$ 
at the noise level, so that Theorem~\ref{thm:multi-stage-featsel} can be applied. 
In this case, Theorem~\ref{thm:multi-stage-featsel} implies that multi-stage capped-$L_1$ regularization achieves
exact recovery of the support set $\Fr(\bar{\bw})$.
In comparison, Lasso does not achieve exact sparse recovery under RIP conditions. While running Lasso followed by
thresholding small coefficients to zero (or using adaptive Lasso of \cite{Zou06} or the two-stage procedure of \cite{ZouLi08}) may achieve exact recovery, such a procedure
requires the condition that 
\begin{equation}
\min_{\bw_j \in \bar{F}} |\bw_j|  \geq c' \sigma \sqrt{\bar{k}\ln p/n} \label{eq:min-bias}
\end{equation}
for some constant $c'$ (also depends on the RIP condition).
This extra $\sqrt{\bar{k}}$ factor is referred to as the bias of the Lasso procedure in \cite{Zhang10-mc+}.
Moreover, it is known that for exact recovery to hold, the requirement of
$\min_{\bw_j \in \bar{F}} |\bw_j| \geq c \sigma \sqrt{\ln p/n}$
(up to a constant) is necessary for all statistical procedures, in the sense that
if $\min_{\bw_j \in \bar{F}} |\bw_j| \leq c' \sigma \sqrt{\ln p/n}$ for a sufficiently small constant $c'$ (under appropriate RIP 
conditions), then no statistical procedure can achieve exact recovery with large probability. 
Therefore statistical procedures that can achieve exact support recovery under (\ref{eq:min-bias}) are referred to as
(nearly) unbiased feature selection methods in \cite{Zhang10-mc+}. 
Theorem~\ref{thm:multi-stage-featsel} shows that multi-stage convex relaxation
with capped-$L_1$ regularization achieves unbiased feature selection.

Results most comparable to what we have obtained here are that of the FoBa procedure
in \cite{Zhang08-foba} and that of the MC+ procedure in
\cite{Zhang10-mc+}. Both can be regarded as (approximate) optimization methods for nonconvex formulations.
The former is a forward backward greedy algorithm, which does not
optimize (\ref{eq:sparse-reg}), while the latter is
a path-following algorithm for solving formulations similar to (\ref{eq:sparse-reg}).
Although results in \cite{Zhang10-mc+} are comparable to ours, we
should note that unlike our procedure, which is efficient due to the finite number of convex optimization, there 
is no proof showing that the path-following
strategy in \cite{Zhang10-mc+} is always efficient (in the sense that there may be
exponentially many switching points). 

\section{Simulation Study}

Numerical examples can be found in \cite{Zhang09-multistage} that demonstrate the advantage of multi-stage convex relaxation
over Lasso. Therefore we shall not repeat a comprehensive study. Nevertheless, this section presents a simple simulation
study to illustrate the theoretical results. 
The $n \times p$ design matrix $X$ is generated with iid random Gaussian entries and each column is normalized
with 2-norm $\sqrt{n}$. Here $n=100$ and $p=250$. We then generate a vector
$\bar{\bw}$ with $\bar{k}=30$ nonzero coefficients, and each nonzero coefficient is uniformly generated from the interval
$(1,10)$. The observation is $y=X \bar{\bw} + \epsilon$, where $\epsilon$ is zero-mean iid Gaussian noise with
standard deviation $\sigma=1$.
We study the feature selection performance of
Multi-stage convex relaxation method in Figure~\ref{fig:multi-stage-sparse}
using various configurations of $\lambda=\tau \sigma \sqrt{\ln(p)/n}$ (with $\tau=1,2,4,8,16,32$),
and $\theta=\mu \lambda$ for various constants $\mu=0.5,1,2,4$. 

The experiments are repeated for 100 times, and 
Table~\ref{tab:performance} reports the probability (percentage in the 100 runs) 
of exact support recovery for each configuration at various stages $\ell$. 
Note that $\ell=1$ corresponds to Lasso and $\ell=2$ is an adaptive Lasso like
two stage method \cite{Zou06,ZouLi08}. The main purpose of this study is to illustrate
that it is beneficial to use more than two stages, as predicted by our theory. 
However, since only $O(\ln(\bar{k}))$ is sufficient, optimal results can be achieved with relatively small number
of stages.
These conclusions can be clearly seen from Table~\ref{tab:performance}. Specifically the results for $\ell=2$ are better than
those of $\ell=1$ (standard Lasso), 
while results of $\ell=4$ are better than those of $\ell=2$. Although the performance of $\ell=8$ is even better,
the improve over $\ell=4$ is small at the optimal configuration of $\lambda$ and $\theta$. 
This is consistent with our theory, which implies that a relatively small number of stages is needed to
achieve good performance. 

\begin{table}[htp]
  \centering
  \begin{tabular}{|c|c|c|c|c|c|c|c|}
\hline
\multicolumn{7}{|c|}{$\theta=0.5 \lambda $} \\ \hline 
$\lambda$ & $ 0.23$ & $ 0.47$ & $ 0.94$ & $ 1.9$ & $ 3.8$ & $ 7.5$\\ \hline
$\ell=1$ & $  0$ & $  0$ & $  0$ & $  0$ & $  0$ & $  0$\\ \hline
$\ell=2$ & $  0$ & $  0$ & $ 0.02$ & $  0$ & $  0$ & $  0$\\ \hline
$\ell=4$ & $  0$ & $ 0.05$ & $ 0.63$ & $ 0.18$ & $  0$ & $  0$\\ \hline
$\ell=8$ & $  0$ & $ 0.12$ & $ 0.83$ & $ 0.25$ & $  0$ & $  0$\\ \hline
\hline
\multicolumn{7}{|c|}{$\theta= 1 \lambda $} \\ \hline 
$\lambda$ & $ 0.23$ & $ 0.47$ & $ 0.94$ & $ 1.9$ & $ 3.8$ & $ 7.5$\\ \hline
$\ell=1$ & $  0$ & $  0$ & $  0$ & $  0$ & $  0$ & $  0$\\ \hline
$\ell=2$ & $  0$ & $ 0.04$ & $ 0.15$ & $ 0.06$ & $  0$ & $  0$\\ \hline
$\ell=4$ & $  0$ & $ 0.33$ & $ 0.86$ & $ 0.13$ & $  0$ & $  0$\\ \hline
$\ell=8$ & $  0$ & $ 0.38$ & $ 0.93$ & $ 0.16$ & $  0$ & $  0$\\ \hline
\hline
\multicolumn{7}{|c|}{$\theta= 2 \lambda $} \\ \hline 
$\lambda$ & $ 0.23$ & $ 0.47$ & $ 0.94$ & $ 1.9$ & $ 3.8$ & $ 7.5$\\ \hline
$\ell=1$ & $  0$ & $  0$ & $  0$ & $  0$ & $  0$ & $  0$\\ \hline
$\ell=2$ & $  0$ & $ 0.14$ & $ 0.22$ & $  0$ & $  0$ & $  0$\\ \hline
$\ell=4$ & $  0$ & $ 0.29$ & $ 0.6$ & $ 0.02$ & $  0$ & $  0$\\ \hline
$\ell=8$ & $  0$ & $ 0.3$ & $ 0.62$ & $ 0.02$ & $  0$ & $  0$\\ \hline
\hline
\multicolumn{7}{|c|}{$\theta= 4 \lambda $} \\ \hline 
$\lambda$ & $ 0.23$ & $ 0.47$ & $ 0.94$ & $ 1.9$ & $ 3.8$ & $ 7.5$\\ \hline
$\ell=1$ & $  0$ & $  0$ & $  0$ & $  0$ & $  0$ & $  0$\\ \hline
$\ell=2$ & $  0$ & $ 0.01$ & $ 0.01$ & $  0$ & $  0$ & $  0$\\ \hline
$\ell=4$ & $  0$ & $ 0.06$ & $ 0.06$ & $  0$ & $  0$ & $  0$\\ \hline
$\ell=8$ & $  0$ & $ 0.06$ & $ 0.06$ & $  0$ & $  0$ & $  0$\\ \hline
  \end{tabular}
  \caption{Probability of Exact Support Recovery for Multi-stage Convex Relaxation}
  \label{tab:performance}
\end{table}

\section{Proof of Theorem~\ref{thm:multi-stage-featsel}}
\label{apx:proof}
The analysis is an adaptation of \cite{Zhang09-multistage}.
While the main proof structure is similar, there are nevertheless
subtle and important differences in the details, and hence a complete proof is still necessary. 
The main technical differences are as follows.
The proof of \cite{Zhang09-multistage} tracks the progress from one stage $\ell-1$ to the next
stage $\ell$ using a bound on 2-norm parameter estimate,
while in the current proof we track the progress using the set of variables that differ significantly from the true variables.
Moreover, in \cite{Zhang09-multistage}, we compare the current estimated parameter to the true
parameter $\bar{\bw}$, which is sufficient for parameter estimation. However, in order to establish 
feature selection result of this paper, it is necessary to compare the  current estimated parameter to 
the least squares solution $\tilde{\bw}$ within the true feature set $\bar{F}$ as defined below in (\ref{eq:ls}). 
These subtle technical differences mean that many details in the proofs presented below differ 
from that of  \cite{Zhang09-multistage}.

\subsection{Auxiliary lemmas}
 
We first introduce some definitions.
Consider the positive semi-definite matrix $A= n^{-1} X^\top X \in \R^{d \times d}$.
Given $s,k \geq 1$ such that $s + k \leq d$.
Let $I,J$ be disjoint subsets of $\{1,\ldots,d\}$
with $k$ and $s$ elements respectively.
Let $A_{I,I} \in R^{k \times k}$ be the restriction of $A$ to indices $I$, $A_{I,J} \in R^{k \times s}$ 
be the restriction of $A$ to indices $I$ on the left and $J$ on the right. 
Similarly we define restriction $\bw_I$ of a vector $\bw \in R^p$ on $I$;
and for convenience, we allow either $\bw_I \in R^k$ or $\bw_I \in R^p$ (where components not in $I$ are zeros) depending on the context.
 
We also need the following quantity in our analysis:
\[
\pi(k,s)=\sup_{\bv  \in R^{k}, \bu \in R^s ,I,J}
  \frac{\bv^\top A_{I,J} \bu \|\bv\|_2}{\bv^\top A_{I,I} \bv  \|\bu\|_\infty } .
\]
The following two lemmas are taken from \cite{Zhang07-l1}.
We skip the proof.
\begin{lemma}\label{lem:gamma-bound}
The following inequality holds:
\[
\pi(k,s) \leq 
\frac{s^{1/2}}{2} \sqrt{\rho_+(s)/\rho_-(k+s)-1} ,
\]
\end{lemma}

\begin{lemma}\label{lem:inner-prod}
Consider $k,s > 0$ and $G \subset \{1,\ldots,d\}$
such that $|G^c|=k$.
Given any $\bw \in R^p$.
Let $J$ be the indices of the $s$ largest components of
$\bw_G$ (in absolute values), and $I= G^c \cup J$.
Then
\[
\max(0,\bw_I^\top A \bw) \geq 
\rho_-(k+s) (\|\bw_I\|_2 - \pi(k+s,s)\|\bw_G\|_1/s) \|\bw_I\|_2 .
\]
\end{lemma}

Our analysis requires us to keep track of progress with respect to the least squares solution $\tilde{\bw}$
with the true feature set $\bar{F}$, which we define below:
\begin{equation}
\tilde{\bw} = \arg\min_{\bw \in R^p} \|X \bw - \by\|_2^2  \qquad \text{subject to } \Fr(\bw) \subset \bar{F} , \label{eq:ls}
\end{equation}
where $\bar{F}=\Fr(\bar{\bw})$. 

The following lemmas require varying degrees of 
modifications from similar lemmas in \cite{Zhang09-multistage}, and thus the
proofs are included for completeness.
\begin{lemma} \label{lem:sub-Gaussian-infnorm}
  Define $\hat{\epsilon}= \frac{1}{n} X^\top (X \tilde{\bw}-\by)$. 
  Under the conditions of Assumption~\ref{assump:fixed}, 
  with probability larger than $1-\eta$:
  \[
  \forall j \in \bar{F}: \quad |\hat{\epsilon}_j|=0 , \quad
  |\tilde{\bw}_j- \bar{\bw}_j| \leq \sigma \sqrt{2 \rho_-(\bar{k})^{-1} \ln (2p/\eta)/ n}  ,
  \]
  and 
  \[
  \forall j \notin \bar{F}: \quad |\hat{\epsilon}_j|\leq \sigma \sqrt{2 \rho_+(1) \ln (2p/\eta)/ n}  .
  \]
\end{lemma}
\begin{proof}
Let $\tilde{P}$ be the projection matrix to the subspace spanned by columns of $X$ in $\bar{F}$, then
we know that
\[
X \tilde{\bw} = \tilde{P} \by
\]
and 
\[
(I-\tilde{P}) \rE \by = \rE \by -  X \bar{\bw} = 0 .
\]
Therefore for each $j$
\[
n |\hat{\epsilon}_j| =  |X_j^\top (X \tilde{\bw}-\by) | = 
|X_j^\top (I-\tilde{P}) (\by- \rE \by)) | .
\]
It implies that $\hat{\epsilon}_j=0$ if $j \in \bar{F}$.
Since for each $j$: the column $X_j$ satisfies
$\|X_j^\top (I-\tilde{P})\|_2^2 \leq n \rho_+(1)$, we have from sub-Gaussian tail bound that
for all $j \notin \bar{F}$ and $\epsilon>0$:
\[
P \left[ |\hat{\epsilon}_j| \geq \epsilon \right] \leq 2 \exp [ - n\epsilon^2/(2 \sigma^2 \rho_+(1)) ] .
\]
Moreover, for each $j \in \bar{F}$, we have
\[
|\tilde{\bw}_j-\bar{\bw}_j|= e_j^\top (X_{\bar{F}}^\top X_{\bar{F}})^{-1} X_{\bar{F}}^\top (\by - \rE \by) .
\]
Since $\|e_j^\top (X_{\bar{F}}^\top X_{\bar{F}})^{-1} X_{\bar{F}}^\top\|_2^2= e_j^\top (X_{\bar{F}}^\top X_{\bar{F}})^{-1} e_j
\leq n^{-1} \rho_-(\bar{k})^{-1}$, we have for all $\epsilon>0$:
\[
P \left[ |\tilde{\bw}_j-\bar{\bw}_j| \geq \epsilon \right] \leq 2 \exp [ - n \rho_-(\bar{k})\epsilon^2/(2 n\sigma^2) ] .
\]
Taking union bound for $j=1,\ldots,d$ (each with probability $\eta/d$) we obtain the desired inequality.
\end{proof}

\begin{lemma}
  \label{lem:L1-nonsparse-dr}
 Consider $G\subset \{1,\ldots,d\}$ such that $\bar{F}\cap G = \emptyset$. 
  Let $\hat{\bw}=\hat{\bw}^{(\ell)}$ be the solution of
  (\ref{eq:convex-relax-L1}), and let $\Delta \hat{\bw}= \hat{\bw}-\tilde{\bw}$.
  Let $\lambda_G = \min_{j \in G} \lambda_j^{(\ell-1)}$ 
  and  $\lambda_0 = \max_{j} \lambda_j^{(\ell-1)}$. 
  If $2\|\hat{\epsilon}\|_\infty\| < \lambda_G$, then
 \[
  \sum_{j \in G} |\hat{\bw}_j |
  \leq 
  \frac{2 \|\hat{\epsilon}\|_\infty}{\lambda_G - 2\|\hat{\epsilon}\|_\infty}
  \sum_{j \notin \bar{F}\cup G} |\hat{\bw}_j|
  +
  \frac{\lambda_0}{\lambda_G - 2\|\hat{\epsilon}\|_\infty}  \sum_{j \in \bar{F}} |\Delta \hat{\bw}_j|
  \leq   \frac{\lambda_0}{\lambda_G - 2\|\hat{\epsilon}\|_\infty}\|\Delta \hat{\bw}_{G^c}\|_1
  .
  \]
\end{lemma}
\begin{proof}
  For simplicity, let $\lambda_j=\lambda_j^{(\ell-1)}$.
  The first order equation implies that 
  \[
  \frac{1}{n} \sum_{i=1}^n 2 (\bx_i^\top \hat{\bw} - y_i) \bx_{i,j}
  + \lambda_j \sgn(\hat{\bw}_j) = 0 ,
  \]
  where $\sgn(\bw_j)=1$ when $\bw_j>0$, $\sgn(\bw_j)=-1$ when $\bw_j<0$,
  and $\sgn(\bw_j) \in [-1,1]$ when $\bw_j=0$.
  This implies that for all $\bv \in \R^p$, we have
  \begin{equation}
    2 \bv^\top A \Delta \hat{\bw} 
    \leq - 2 \bv^\top \hat{\epsilon} 
    - \sum_{j=1}^p \lambda_j \bv_j \sgn(\hat{\bw}_j) . \label{eq:dr}
  \end{equation}
Now, let $\bv=\Delta \hat{\bw}$ in (\ref{eq:dr}), and notice that $\hat{\epsilon}_{\bar{F}}=0$, we obtain
\begin{align*}
0 \leq&  2 \Delta \hat{\bw}^\top A \Delta \hat{\bw} 
\leq 2 |\Delta \hat{\bw}^\top \hat{\epsilon} |
- \sum_{j=1}^p \lambda_j \Delta \hat{\bw}_j \sgn(\hat{\bw}_j) \\
\leq& 2 \|\Delta \hat{\bw}_{\bar{F}^c}\|_1 \|\hat{\epsilon} \|_\infty
- \sum_{j \in \bar{F}} \lambda_j \Delta \hat{\bw}_j \sgn(\hat{\bw}_j) 
- \sum_{j \notin \bar{F}} \lambda_j \Delta \hat{\bw}_j \sgn(\hat{\bw}_j) \\
\leq& 2 \|\Delta \hat{\bw}_{\bar{F}^c}\|_1 \|\hat{\epsilon} \|_\infty
+ \sum_{j \in \bar{F}} \lambda_j |\Delta \hat{\bw}_j|
- \sum_{j \notin \bar{F}} \lambda_j |\hat{\bw}_j| \\
\leq& \sum_{j \in G} (2\|\hat{\epsilon}\|_\infty -\lambda_G) |\hat{\bw}_j |
+ \sum_{j \notin G\cup \bar{F}} 2 \|\hat{\epsilon}\|_\infty |\hat{\bw}_j|
+ \sum_{j \in \bar{F}} \lambda_0 |\Delta \hat{\bw}_j| .
\end{align*}
By rearranging the above inequality, 
we obtain the first desired bound. The second inequality uses $2\|\hat{\epsilon}\|_\infty \leq \lambda_0$.
\end{proof}

\begin{lemma}\label{lem:L1-nonsparse3}
  Using the notations of Lemma~\ref{lem:L1-nonsparse-dr}, and
  let $J$ be the indices of the largest $s$ coefficients (in absolute value)
  of $\hat{\bw}_G$.
  Let $I= G^c \cup J$ and $k=|G^c|$.
  If   $0 \leq \lambda_0 /(\lambda_G -2\|\hat{\epsilon}\|_\infty) \leq 3$, then
  \[
 \|\Delta \hat{\bw}\|_2
 \leq (1+ (3 k/s)^{0.5}) \|\Delta \hat{\bw}_I\|_2 .
  \]
\end{lemma}
\begin{proof}
  Using
  $\lambda_0/(\lambda_G-2\|\hat{\epsilon}\|_\infty)
  \leq 3$,
  we obtain from Lemma~\ref{lem:L1-nonsparse-dr} 
  \[
  \|\hat{\bw}_G\|_1 \leq 3 \|\Delta \hat{\bw} - \hat{\bw}_G \|_1 . 
  \]
  Therefore
\begin{align*}
\|\Delta \hat{\bw}-\Delta \hat{\bw}_I\|_\infty \leq& \|\Delta \hat{\bw}_J\|_1/s \\
=& s^{-1}
[\|\Delta \hat{\bw}_G\|_1 - \|\Delta \hat{\bw}-\Delta \hat{\bw}_I\|_1] \\
\leq& 
s^{-1} [3\|\Delta \hat{\bw} - \hat{\bw}_G\|_1 - \|\Delta \hat{\bw}-\Delta \hat{\bw}_I\|_1],
\end{align*}
which implies that
\begin{align*}
\|\Delta \hat{\bw}-\Delta \hat{\bw}_I\|_2 \leq &
(\|\Delta \hat{\bw}-\Delta \hat{\bw}_I\|_1
\|\Delta \hat{\bw}-\Delta \hat{\bw}_I\|_\infty)^{1/2} \\
\leq& 
\left[\|\Delta \hat{\bw}-\Delta \hat{\bw}_I\|_1
(3\|\Delta \hat{\bw} - \hat{\bw}_G\|_1 - \|\Delta \hat{\bw}-\Delta \hat{\bw}_I\|_1)\right]^{1/2} s^{-1/2} \\
\leq& 
\left[
(3\|\Delta \hat{\bw} - \hat{\bw}_G\|_1/2)^2\right]^{1/2} s^{-1/2} \\
\leq& (3/2) s^{-1/2} \|\Delta \hat{\bw} - \hat{\bw}_G\|_1 \\
\leq& (3/2) s^{-1/2} \bar{k}^{1/2} \|\Delta \hat{\bw} - \hat{\bw}_G\|_2
\leq (3 k/s)^{1/2} \|\Delta \hat{\bw}_I\|_2 .
\end{align*}
The third inequality uses the simple algebraic inequality $a(3b-a) \leq (3b/2)^2$.
By rearranging this inequality, we obtain the desired bound.
Note that in the above derivation, we have used the fact that
$\bar{F} \cap G=\emptyset$, which implies that $\Delta \hat{\bw}_G=\hat{\bw}_G$, and thus
$\Delta \hat{\bw} - \hat{\bw}_G=\Delta\hat{\bw}_{G^c}$.
\end{proof}

\begin{lemma}\label{lem:L1-nonsparse1}
  Let the conditions of Lemma~\ref{lem:L1-nonsparse-dr} and Lemma~\ref{lem:L1-nonsparse3} hold, and
  let $k=|G^c|$.
  If $t=1-\pi(k+s,s) k^{1/2} s^{-1} \in (0,4/3)$, and
  $0 \leq \lambda_0/(\lambda_G -2\|\hat{\epsilon}\|_\infty) \leq (4-t)/(4-3t)$, then
\[
   \|\Delta \hat{\bw}\|_2
   \leq (1+ (3k/s)^{0.5}) \|\Delta \hat{\bw}_I\|_2
  \leq \frac{1+ (3k/s)^{0.5}}{t \rho_-(k+s)}
  \left[
 2 \|\hat{\epsilon}_{G^c}\|_2
+ \left(\sum_{j \in \bar{F}} (\lambda_j^{(\ell-1)})^2\right)^{1/2}\right] .
\]
\end{lemma}
\begin{proof}
  Let $J$ be the indices of the largest $s$ coefficients (in absolute value)
  of $\hat{\bw}_G$,
  and $I= G^c \cup J$. 
The conditions of the lemma imply that
\begin{align*}
  \max(0,\Delta \hat{\bw}_I^\top A \Delta \hat{\bw}) 
\geq&
\rho_-(k+s)
  [\|\Delta \hat{\bw}_I\|_2 -  \pi(k+s,s) \|\hat{\bw}_G\|_1/s] \|\Delta \hat{\bw}_I\|_2 \\
\geq&
\rho_-(k+s)
  [1 -  (1-t)(4-t)(4-3t)^{-1} ] \|\Delta \hat{\bw}_I\|_2^2 \\
  \geq& 0.5 t \rho_-(k+s) \|\Delta \hat{\bw}_I\|_2^2 .
\end{align*}
In the above derivation,
the first inequality is due to Lemma~\ref{lem:inner-prod};
the second inequality is due to the conditions of this lemma plus Lemma~\ref{lem:L1-nonsparse-dr},
which implies that 
\[
\|\hat{\bw}_G\|_1 \leq 
\frac{\lambda_0}{\lambda_G
  -2\|\hat{\epsilon}\|_\infty} \|\hat{\bw}_{G^c}\|_1\leq 
\frac{\lambda_0}{\lambda_G
  -2\|\hat{\epsilon}\|_\infty} 
\sqrt{k}\|\hat{\bw}_{I}\|_2 ;
\]
and the last inequality follows from $1-(1-t)(4-t)(4-3t)^{-1} \geq 0.5 t$, which holds for
$t \in (0,4/3)$.

If $\Delta \hat{\bw}_I^\top A \Delta \hat{\bw} \leq 0$,
then the above inequality, together with Lemma~\ref{lem:L1-nonsparse3},
imply the lemma.
Therefore in the following, we can assume that 
\[
\Delta \hat{\bw}_I^\top A \Delta \hat{\bw}
  \geq 0.5 t \rho_-(k+s) \|\Delta \hat{\bw}_I\|_2^2 .
\]
Moreover, let $\lambda_j=\lambda_j^{(\ell-1)}$.
We obtain from (\ref{eq:dr}) with $\bv=\Delta \hat{\bw}_I$ the 
following:
\begin{align*}
&  2 \Delta \hat{\bw}_I^\top A \Delta \hat{\bw} 
\leq - 2 \Delta \hat{\bw}_I^\top \hat{\epsilon} 
    - \sum_{j \in I} \lambda_j \Delta \hat{\bw}_j \sgn(\hat{\bw}_j) \\
= & - 2 \Delta \hat{\bw}_I^\top \hat{\epsilon}_{G^c}
-  2 \Delta \hat{\bw}_I^\top \hat{\epsilon}_{G} 
- \sum_{j \in \bar{F}} \lambda_j \Delta \hat{\bw}_j \sgn(\hat{\bw}_j) 
-\sum_{j \in G} \lambda_j |\Delta \hat{\bw}_j| 
-\sum_{j \in \bar{F}^c \cap G^c} \lambda_j |\Delta \hat{\bw}_j| \\
\leq & 2 \|\Delta \hat{\bw}_I\|_2 \|\hat{\epsilon}_{G^c}\|_2
+2 \|\hat{\epsilon}_{G}\|_\infty 
\sum_{j \in G} |\Delta \hat{\bw}_j| 
+ \sum_{j \in \bar{F}} \lambda_j |\Delta \hat{\bw}_j| 
-\sum_{j \in G} \lambda_j |\Delta \hat{\bw}_j| \\
\leq & 2 \|\Delta \hat{\bw}_I\|_2 \|\hat{\epsilon}_{G^c}\|_2
+ (\sum_{j \in \bar{F}} \lambda_j^2)^{1/2} \|\Delta \hat{\bw}_I\|_2 .
\end{align*}
Note that the equality uses the fact that $G \subset \bar{F}^c$, and 
$\Delta \hat{\bw}_j \sgn(\hat{\bw}_j) =|\hat{\bw}_j|$ for $j \in \bar{F}^c$.
The last inequality uses the fact that $\forall j \in G$: $\lambda_j \geq \lambda_G \geq 2 \|\hat{\epsilon}_G\|_\infty$.
Now by combining the above two estimates, we obtain
\[
   \|\Delta \hat{\bw}_I\|_2
  \leq \frac{1}{t \rho_-(k+s)}
  \left[
 2 \|\hat{\epsilon}_{G^c}\|_2
+ (\sum_{j \in \bar{F}} \lambda_j^2)^{1/2} \right] .
\]
The desired bound follows from Lemma~\ref{lem:L1-nonsparse3}.
\end{proof}

\begin{lemma} \label{lem:lambda-bound}
 Let $\lambda_j= \lambda I(|{\bw}_j|\leq \theta)$ for some ${\bw} \in R^p$,
  then 
  \[
  \left(\sum_{j \in \bar{F}} \lambda_j^2\right)^{1/2}
  \leq  \lambda \sqrt{\sum_{j \in \bar{F}} I(|\bar{\bw}_j| \leq 2\theta)}
  + \lambda 
  \left|\{j \in \bar{F}: |\bar{\bw}_j-{\bw}_j|\geq \theta \}\right|^{1/2} .
  \]
\end{lemma}
\begin{proof}
  By assumption, if $|\bar{\bw}_j-{\bw}_j| \geq \theta$, then
  \[
  I(|{\bw}_j| \leq \theta) \leq 1 \leq  I(|\bar{\bw}_j-{\bw}_j|\geq \theta) ;
  \]
  otherwise, 
  $I(|{\bw}_j| \leq \theta) \leq I(|\bar{\bw}_j| \leq 2\theta)$. It follows that
  the following inequality always holds:
  \[
  I(|{\bw}_j|\leq \theta) \leq
  I(|\bar{\bw}_j| \leq 2\theta) + I(|\bar{\bw}_j-{\bw}_j|\geq \theta) .
  \]
  The desired bound is a direct consequence of the above result and
  the 2-norm triangle inequality 
  \[
  (\sum_j (x_j + \Delta x_j)^2)^{1/2} \leq (\sum_j x_j^2)^{1/2} + (\sum_j \Delta x_j^2)^{1/2} .
  \]
\end{proof}

\begin{lemma} \label{lem:bound-ell=1}
  Define $F^{(\ell)} = \{j: |\hat{\bw}_j^{(\ell)} -\bar{\bw}_j| \geq \theta\}$.
  Under the conditions of  Theorem~\ref{thm:multi-stage-featsel}, we have
  for all $s \geq 2 \bar{k}$:
  \[
  \|\hat{\bw}^{(\ell)}- \tilde{\bw}\|_2
  \leq \frac{5.7 \lambda}{\rho_-(1.5\bar{k}+s)} \sqrt{|F^{(\ell-1)}|} ,
  \]
  and
  \[
  \sqrt{|F^{(\ell)}|}\leq  \frac{6 \lambda \theta^{-1}}{\rho_-(1.5\bar{k}+s)} \sqrt{|F^{(\ell-1)}|} .
  \]
\end{lemma}
\begin{proof}
For all $t \in [0.5,4/3)$, by using Lemma~\ref{lem:sub-Gaussian-infnorm}, we know that
the condition of the theorem implies that 
\[
\frac{\lambda}
{\lambda - 2 \|\hat{\epsilon}\|_\infty}
\leq 7/5 \leq \frac{4-t}{4-3t} .
\]
Moreover, Lemma~\ref{lem:gamma-bound} implies that the condition
\[
0.5 \leq t=1 - \pi(1.5\bar{k}+s,s) (1.5\bar{k})^{0.5} /s
\]
is also satisfied.
This means that the conditions of Lemma~\ref{lem:L1-nonsparse1} (with $\lambda_0=\lambda_G=\lambda$)
are satisfied. 

Now, we assume that at some $\ell \geq 1$, 
\begin{equation}
|G_\ell^c| \leq 1.5 \bar{k}, \quad \text{where } G_\ell=\{j \notin \bar{F}:
\lambda_j^{(\ell-1)} =\lambda \}  , \label{eq:G}
\end{equation}
then it is easy to verify that $G_\ell^c \setminus \bar{F} \subset F^{(\ell-1)}$.

Moreover, with the definition of $G=G_\ell$
in Lemma~\ref{lem:L1-nonsparse1} 
and Lemma~\ref{lem:lambda-bound}, 
we can set $\lambda_0=\lambda_G=\lambda$ and obtain (note also that $\hat{\epsilon}_{\bar{F}}=0$)
\begin{align*}
   \|\hat{\bw}^{(\ell)}-\tilde{\bw}\|_2
  \leq& \frac{1+\sqrt{3}}{t \rho_-(1.5\bar{k}+s)}
\left[ 2\|\hat{\epsilon}_{G_\ell^c \setminus\bar{F}}\|_2 +
\left(\sum_{j \in \bar{F}} (\lambda_j^{(\ell-1)})^2\right)^{1/2}  \right]\\
\leq& \frac{1+\sqrt{3}}{t \rho_-(1.5\bar{k}+s)}
\left[ 2 \sqrt{| F^{(\ell-1)}\setminus \bar{F}|}  \|\hat{\epsilon}\|_\infty +
\sqrt{|F^{(\ell-1)}\cap \bar{F}|} \lambda  \right]\\
\leq& \frac{1+\sqrt{3}}{t \rho_-(1.5\bar{k}+s)}
\left[ (2/7) \sqrt{| F^{(\ell-1)}\setminus \bar{F}|}  +
\sqrt{|F^{(\ell-1)}\cap \bar{F}|}  \right] \lambda \\
 \leq& \frac{1+\sqrt{3}}{0.5 \rho_-(1.5\bar{k}+s)}
\left[ \sqrt{1.082 |F^{(\ell-1)}|} \right]\lambda  \\
\leq& \frac{5.7 \lambda}{\rho_-(1.5\bar{k}+s)} \sqrt{|F^{(\ell-1)}|} ,
\end{align*}
where the first inequality is due to Lemma~\ref{lem:L1-nonsparse1}. 
The second inequality uses the facts that $G_\ell^c \setminus \bar{F} \subset F^{(\ell-1)} \setminus \bar{F}$, 
and Lemma~\ref{lem:lambda-bound} with $I(|\bar{\bw}_j| \leq 2\theta)=0$ (for all $j \in \bar{F}$).
The third inequality uses $2\|\hat{\epsilon}\|_\infty \leq (2/7)\lambda$, and the
fourth inequality uses $(2/7)a + b \leq \sqrt{1.082(a^2+b^2)}$.

Since Lemma~\ref{lem:sub-Gaussian-infnorm} implies that
\[
\|\tilde{\bw}-\bar{\bw}\|_\infty \leq (1/7) \lambda /\sqrt{\rho_+(1) \rho_-(\bar{k})} ,
\]
we know that $j \in F^{(\ell)}$ implies that
\[
|\tilde{\bw}_j-\hat{\bw}^{(\ell)}_j|\geq \theta - (1/7) \lambda /\sqrt{\rho_+(1) \rho_-(\bar{k})} \geq (41/42) \theta .
\]
Therefore
\begin{align*}
\sqrt{|F^{(\ell)}|}
\leq& (41\theta/42)^{-1} \|\tilde{\bw}-\hat{\bw}^{(\ell)}\|_2 \\
\leq& \frac{5.7 \lambda (41\theta/42)^{-1}}{\rho_-(1.5\bar{k}+s)} \sqrt{|F^{(\ell-1)}|} \\
\leq& \frac{6 \lambda \theta^{-1}}{\rho_-(1.5\bar{k}+s)}  \sqrt{|F^{(\ell-1)}|} .
\end{align*}
That is, under the assumption of (\ref{eq:G}), the lemma holds at $\ell$.

Therefore next we only need to prove by induction on $\ell$  that
(\ref{eq:G}) holds for all $\ell=1, 2,\ldots$.
When $\ell=1$, we have $G_1^c=\bar{F}$, which implies that (\ref{eq:G}) holds.

Now assume that (\ref{eq:G}) holds at $\ell$ for some $\ell \geq 1$. 
Then by the induction hypothesis we know that the lemma holds at $\ell$.
This means that
\begin{align*}
\sqrt{|G_{\ell+1}^c \setminus \bar{F}|} \leq & \sqrt{|F^{(\ell)}|} \\
\leq& \frac{6 \lambda \theta^{-1}}{\rho_-(1.5\bar{k}+s)} \sqrt{|F^{(\ell-1)}|} \\
\leq& \sqrt{0.5 |F^{(\ell-1)}|} \\
\leq& \cdots \leq 0.5^{\ell/2} |F^{(0)}| .
\end{align*}
The first inequality is due to the fact $G_{\ell+1}^c \setminus \bar{F} \subset F^{(\ell)}$.
The second inequality uses the assumption of $\theta$ in the theorem.
The last inequality uses induction. Now note that $F^{(0)}=\bar{F}$, we thus have
$|G_{\ell+1}^c \setminus \bar{F}| \leq 0.5\bar{k}$.
This completes the induction step.
\end{proof}

\subsection{Proof of Theorem~\ref{thm:multi-stage-featsel}}
Define 
\[
\beta=\frac{6 \lambda \theta^{-1}}{\rho_-(1.5\bar{k}+s)} ,
\]
We have $\beta<1$ by the assumption of the theorem.
Using induction, we have from Lemma~\ref{lem:bound-ell=1} that
\begin{align*}
  \sqrt{|F^{(L)}|}\leq&  \beta \sqrt{|F^{(L-1)}|}  \\
\leq& \cdots \\
\leq& \beta^{L} \sqrt{|F^{(0)}|} \\
\leq& \beta^{L} \sqrt{\bar{k}} < 1 .
\end{align*}
This means that when $\ell > L$, $|F^{(\ell-1)}|=0$.
Therefore by applying Lemma~\ref{lem:bound-ell=1} again we obtain
\[
  \|\hat{\bw}^{(\ell)}- \tilde{\bw}\|_2 =0 .
\]
Since Lemma~\ref{lem:sub-Gaussian-infnorm} implies that
\[
\|\tilde{\bw}-\bar{\bw}\|_\infty \leq (1/7) \lambda /\sqrt{\rho_+(1) \rho_-(\bar{k})} < \theta ,
\]
we have
\[
\Fr(\tilde{\bw})=\Fr(\bar{\bw}) .
\]
This implies that $\Fr(\hat{\bw}^{(\ell)})=\Fr(\bar{\bw})$.

\section{Discussion}

This paper investigated the performance of multi-stage convex relaxation for feature selection, where it is shown that
under RIP, the procedure can achieve unbiased feature selection.
This result complements that of \cite{Zhang09-multistage} which studies the parameter estimation performance 
of multi-stage convex relaxation.
It also complements similar results obtained in \cite{Zhang08-foba} and \cite{Zhang10-mc+} for different 
computational procedures. One advantage of our result over that in \cite{Zhang10-mc+} is that 
the multi-stage convex relaxation method is provably efficient because the correct feature set can be obtained
after no more than
$O(\log \bar{k})$ number of iterations. In comparison, a computational efficiency statement for the path-following method
of  \cite{Zhang10-mc+} remains open.

\bibliographystyle{plain}
\bibliography{L1,learning,pubj}

\end{document}